%% file: ms.tex
\documentclass[runningheads]{llncs}
\usepackage{graphicx,comment,amsmath}

\input{preamble}
\begin{document}
	\title{Global Convergence of Sobolev Training\\ for Overparameterized Neural Networks}
	\titlerunning{Global Convergence of Sobolev Training}
	%
	\author{Jorio Cocola\inst{1}\and
		Paul Hand\inst{1,2}}
	\authorrunning{J. Cocola and P. Hand}
	%
	\institute{Department of Mathemathics, Northeastern University,
		\and
		Khoury College of Computer Sciences, Northeastern University, \\Boston MA, USA. \\
		\email{\{cocola.j,p.hand\}@northeastern.edu}}
	\maketitle              
	\begin{abstract}
		Sobolev loss is used when training a network to approximate the values and derivatives of a target function at a prescribed set of input points. Recent works have demonstrated its successful applications in 
		various tasks such as distillation or synthetic gradient prediction. In this work we prove that an overparameterized two-layer relu neural network trained on the Sobolev loss with gradient flow from random initialization can fit any given function values and any given directional derivatives,
		under a separation condition on the input data.
	\end{abstract}
	
	\input{sections/intro}

	\input{sections/main_res}

\input{sections/main_proof}

\bibliographystyle{splncs04}
\bibliography{references}

\appendix
\input{sections/appendix}
\input{sections/supplementary_proofs}
\bigskip

\noindent \textbf{Acknowledgements.} PH is supported in part by NSF CAREER Grant DMS-1848087.

\end{document}

%% file: preamble.tex
\usepackage{float, amssymb, bbold}
\usepackage{amsmath,bbm}
\usepackage{xcolor}

\newcommand{\Fb}{\bar{F}}
\newcommand{\Hinfty}{{H}^{\infty}}

\newcommand{\R}{\bbbr}
\newcommand{\tlambda}{{\lambda}_*}

\newcommand{\dd}{\mathrm{d}}
\newcommand{\Nhb}{B_0}
\newcommand{\pa}{\partial}

\newcommand{\Ibb}{\mathbb{I}}

\DeclareMathOperator{\EX}{\mathbb{E}}
\DeclareMathOperator{\PX}{\mathbb{P}}

\spnewtheorem{assumption}{Assumption}{\bfseries}{\itshape}

%% file: sections/intro.tex
\section{Introduction}

Deep neural networks are ubiquitous and have established state of the art performances in a wide variety of applications and fields.  
These networks often have a large number of parameters which are tuned via   gradient descent (or its variants) on an empirical risk minimization task. In particular 
in supervised learning it is often required that the output of the network fits certain values/labels that can be thought as coming from an unknown target function.  
In many settings, though, additional prior information on the task or target function might be available, and enforcing them might be of interest. One such example 
is the case of high order derivatives of the unknown target function, which, as shown in  \cite{czarnecki2017sobolev},  naturally arises in problems such as distillation, in which a large teacher network is used to train a more compact student  
network, or prediction of synthetic gradients for training deep complex models.  Therefore \cite{czarnecki2017sobolev} proposed  the \textit{``Sobolev training"} which given training inputs $\{x_i\}_{i=1}^n$, attempts to minimizes the following empirical risk:
\begin{equation}\label{eq:Sobolev}
	L(W) = \sum_{i=1}^n \big[ \ell(f(W,x_i), f^*(x_i)) + \sum_{j=1}^K \ell_j(D_x^j f(W,x_i),D_x^j f^*(x_i) )	\big]
\end{equation}
where $f(W,x)$ is a neural network with input $x$ and parameters $W$,
$f^*$ denotes the target function, $\ell$ is a loss penalizing the deviation from the outputs of $f$, and $\ell_j$ are loss functions penalizing the deviations 
of the $j$-th derivative $D_x^j f$ of the network $f$ with respect to $x$ from the $j$-th derivative $D_x^j f^*$  of the target $f^*$.

The empirical successes of  Sobolev training have been demonstrated in a number of works. In \cite{czarnecki2017sobolev} it was shown that Sobolev training leads to 
smaller generalization errors than standard training, in tasks such as distillation and synthetic gradient prediction especially in the low data regime. Similar results were also obtained for transfer learning via Jacobian matching in \cite{srinivas2018knowledge}. Earlier Sobolev training was applied 
in \cite{simard1992tangent} in order to enforce invariance to translations and small rotations. More recently, instead,  Sobolev training has been used in the context of anisotropic hyperelasticity in order to  improve the predictions on the stress tensor (derivative of the network with respect to the input deformation tensor) in \cite{vlassis2020geometric}. 
Finally, the idea of Sobolev training is also tightly connected to 
other techniques which have been recently successfully employed, such as attention matching in student distillation \cite{zagoruyko2016paying} and \cite{czarnecki2017sobolev}, or convex data augmentation for generalization and robustness improvement \cite{zhang2017mixup}.

On the theoretical side, justification for Sobolev training was given by 
\cite{czarnecki2017sobolev}, extending the classical work of Hornik  \cite{hornik1991approximation} and giving universal approximation properties of neural networks with relu activation function in Sobolev spaces. This result was then further improved  for deep networks in \cite{guhring2019error}. 
While these works motivated the use of the Sobolev loss \eqref{eq:Sobolev}, 
conditions under which it can be successfully minimized were not given.
In particular even though the network used in Sobolev training are usually shallow, the resulting loss  \eqref{eq:Sobolev} is highly non-convex and therefore the success 
of first order methods is not a priori guaranteed. 

In this paper we study a two-layer relu neural network trained with a Sobolev loss 
when at each input point the output values and a set of directional derivatives of the target function are given. Leveraging recent results on training with standard losses
\cite{arora2019fine,zou2018stochastic,allen2018convergence,zou2018stochastic,oymak2019towards}
we show that if the network is sufficiently overparameterized, the weights are randomly initialized, 
and the data satisfy certain natural non-degeneracy assumptions, Gradient Flow achieves a global minimum.

%% file: sections/main_res.tex
\section{Main Result}
We study the training of neural networks with \textit{``Directional Sobolev Training''}. 
In particular we assume we are given training data
\begin{equation}\label{eq:tr_data}
\big\{{x}_i, y_i, V_i , h_i \big\}_{i=1}^n,
\end{equation}
where $x_i  \in \R^d$, $y_i \in \R$, $h_i \in \R^k$ and $V_i \in \R^{d \times k}$ with orthonormal columns (unit Euclidean norm and pairwise orthogonal).
This training data can be thought as being generated by a differentiable
function $f^*: \R^d \to \R$	according to
\begin{equation}\label{eq:fstar_data}
y_i = f^*(x_i), \quad \text{and}\quad h_i = V_i^T \nabla f^*(x_i) \qquad \text{for}\; i=1, \dots, n
\end{equation}
so that each entry of the vector $h_i$ corresponds to a directional derivative 
of $f^*$ in the direction given by the corresponding column of the matrix $V_i$. 	
We will denote by $y \in \R^n$ and $h \in \R^{n k}$
the vectors with $n$ entries $y_i$ and $n$ blocks $h_i$ respectively.

In this work we study the training of a two-layer neural network with \textit{width}~$m$:
\begin{equation}\label{eq:NN}
f(W,{x}) = \sum_{r = 1}^{m} a_r \sigma( w_r^T x),
\end{equation}
where $a_r$ are fixed at initialization, $\sigma(z) = \max(0, z)$ is the relu activation function, and $W$ is the weight matrix with rows $\{w_r^T\}_{r=1}^{m}$. 
The network  weights $\{w_r\}_{r=1}^{m}$ are learned by minimizing the \textit{Directional Sobolev Loss}
\begin{equation*}
	\min_{W \in R^{m\times d}} L(W) :=  \frac{1}{2} \sum_{i = 1}^{n} (f(W,x_i) - y_i )^2 + \frac{1}{2} \sum_{i = 1}^{n} \| V_i^T \nabla_x f(W,x_i) - h_{i} \|_2^2.
\end{equation*}
via the \textit{Gradient Flow}
\begin{equation}\label{eq:GF}
	\frac{\dd w_r(t)}{\dd t} = - \frac{\pa L(W(t))}{\pa w_r}.
\end{equation}
Note that even though the relu activation function $\sigma$ is not differentiable, we let $\sigma'(z) = \Ibb\{z > 0\}$ and $\sigma''(z) = 0$.  This corresponds to the choice made in most of the deep learning libraries, and  
the dynamical system \eqref{eq:GF} can then be seen as the one followed in practice when using Sobolev training. Explicit formulas for the partial derivatives are given in the next section.

In this work we prove  that for wide enough networks, gradient flow converges to a global minimizer of $L(W)$.
In particular define the vectors of residuals $e(t) \in \R^n$ and $S(t) \in \R^{k\cdot n}$ with coordinates
\begin{equation}\label{eq:residuals}
[e(t)]_{i} = y_i - f(W(t), x_i) \qquad \big[S(t)\big]_i = h_i -  V_i^T \nabla_x f(W(t),x_i).
\end{equation}
We show that $e(t) \to 0$ and $S(t) \to 0$ as $t \to \infty$, under the following assumption of non-degeneracy  of the training data.
\begin{assumption}\label{hyp:parallel_indp}
	The data are normalized so that $\|x_i\|_2 = 1$ and there exist $0 < \delta_1$ and $0 \leq  \delta_2 < k^{-1}$ such that the following hold:
	\begin{equation}\label{eq:sep_x}
		\min_{i \neq j}(\|x_i - x_j\|_2, \|x_i + x_j\|_2) \geq \delta_1,
	\end{equation}
	and for every $i = 1, \dots, n$:
	\begin{equation}\label{eq:sep_v}
		\max_{1 \leq j\leq k} |v_{i, j}^T x_i | \leq \delta_2,
	\end{equation}
	where  $v_{i,j} $ are the columns of $V_i$.
\end{assumption}

 Given $w \in \R^d$ define the following ``feature maps'':
\[
\phi_w(x_i) := \sigma'(w^T x_i) x_i,
\] 
\[
\psi_w(x_i) := \sigma'(w^Tx_i) V_i,
\]
and matrix:
\begin{equation*}
	\Omega(w) = [\phi_w(x_1), \dots, \phi_w(x_n), \psi_w(x_1), \dots \psi_w(x_n)] \;  \in \R^{d \times (k+1)n}.
\end{equation*}
The next quantity plays an important role in the proof of convergence 
of the gradient flow \eqref{eq:GF}.
\begin{definition}
	Define the matrix $H^{\infty} \in \R^{n(k+1)\times n(k+1)}$ with entries given by $[H^{\infty}]_{i,j} = \EX_{w \sim \mathcal{N}(0,I_d)}[\Omega(w)^T\Omega(w)]_{i,j}$, and let $\tlambda$ be its smallest eigenvalue.
\end{definition}
Under the non-degeneracy of the training set we show that $H^{\infty}$ is strictly positive definite.
\begin{proposition}\label{prop:tlambda}
	Under the Assumptions \ref{hyp:parallel_indp} the minimum eigenvalue of $\Hinfty$ obeys:
	\[
		\tlambda \geq \frac{(1 - k \delta_2) \delta_1 }{100\, n^2} .
	\]
\end{proposition}
We are now ready to state the main result of this work.
\begin{theorem}\label{thm:main_thm}
	Suppose Assumption \ref{hyp:parallel_indp} is satisfied.
	Consider a one hidden layer neural network  \eqref{eq:NN}, let $\gamma = \|y\|_2 + \|h\|_2$, 
	set the number of hidden nodes to $m = \Omega({n^6 k^4 \gamma^2}/(\lambda_*^4 \delta^3))$
	and i.i.d. initialize the weights according to:
	\begin{equation}\label{eq:initialize}
	w_r \sim \mathcal{N}(0,I_d)\quad  \text{and} \quad a_r \sim \mathrm{unif}\{- \frac{1}{m^{1/2}},\frac{1}{m^{1/2}} \} \qquad \text{for}\; r = 1, \ldots, m . 
	 \end{equation}
	 Consider the Gradient Flow \eqref{eq:GF}, then with probability $1-\delta$ over the random initialization of $\{w_r\}_r$ and $\{a_r\}_r$, for every $t \geq 0$:
	\[
		\| e(t)\|_2^2 + \| S(t)\|_2^2  \leq \exp(- \tlambda t)\, ( \| e(0)\|_2^2 + \| S(0)\|_2^2)
	\]
	and in particular $L(W(t)) \to 0$ as $t \to \infty$.
\end{theorem}
The proof of this theorem is given in Section \ref{sec:Proof}, below we will show how to extend this result to a network with bias.

\subsection{Consequences for a network with bias}
Given training data \eqref{eq:tr_data} generated by a target function $f^*$ according to \eqref{eq:fstar_data}, in this section we demonstrate how the previous theory can be extended to the Sobolev training of a two-layer network with width $m$ and bias term $b$:
\begin{equation}\label{eq:NN3}
    g({W},b,{{x}}) = \sum_{r = 1}^{m} a_r \sigma(\alpha\, w_r^T {x} + \beta\, b_r)
\end{equation}
where\footnote{Notice that the introduction of the constants $\alpha$ and $\beta$ does not change the expressivity of the network.} $\alpha = 1/(2k)$ and $\beta = \sqrt{1 - \alpha^2}$. 

Similarly as before, the network  weights $\{w_r\}_{r=1}^{m}$ and biases $\{b_r\}_{r=1}^{m}$ are learned by minimizing the \textit{Directional Sobolev Loss}
\begin{equation}\label{eq:SobolevTr2}
	\min_{W \in \R^{m\times d}, b \in \R^{d}} L(W, b) :=  \frac{1}{2} \sum_{i = 1}^{n} (g(W, b, {x}_i) - y_i )^2 + \frac{1}{2} \sum_{i = 1}^{n}  \| \frac{1}{\alpha} V_i^T \nabla_x g(W, b, {x}_i) - \frac{h_{i}}{\alpha} \|_2^2.
\end{equation}
via the \textit{Gradient Flow}
\begin{equation}\label{eq:GF3}
	\frac{\dd w_r(t)}{\dd t} = - \frac{\pa L(W(t))}{\pa w_r} \quad\text{and}\quad \frac{\dd b_r(t)}{\dd t} = - \frac{\pa L(W(t))}{\pa b_r}.
\end{equation}
Based on the following separation conditions on the input point $\{x_i\}_i$ we will prove convergence to zero training error of the Sobolev loss.
\begin{assumption}\label{hyp:parallel_indp3}
	The data are normalized so that $\|x_i\|_2 = 1$ and there exists $\hat{\delta}_1 > 0$ such that the following holds
	\begin{equation}\label{eq:sep_x2}
		\min_{i \neq j}(\|x_i - x_j\|_2) \geq \hat{\delta}_1.
	\end{equation}
\end{assumption}
Define the vectors of residuals $e(t) \in \R^n$ and $S(t) \in \R^{k\cdot n}$ with coordinates
\begin{equation*}
[e(t)]_{i} = y_i - g(W(t), b(t), x_i), \quad \big[S(t)\big]_i = (h_i -  V_i^T \nabla_x g(W(t), b(t),x_i))/\alpha,
\end{equation*}
then the next theorem follows readily from the analysis in the previous section.
\begin{theorem}\label{thm:bias}
Suppose Assumption \ref{hyp:parallel_indp3} is satisfied.
	Consider a two-layer neural network  \eqref{eq:NN3}, let $\gamma = \|y\|_2 + \|h/\alpha\|_2$, 
	set the number of hidden nodes to $m = \Omega({n^6 k^4 \gamma^2}/(\lambda_0^4 \delta^3))$
	and i.i.d. initialize the weights according to:
	\begin{equation*}
	w_r \sim \mathcal{N}(0,I_d), \quad b_r \sim \mathcal{N}(0,1) \quad \text{and} \quad a_r \sim \mathrm{unif}\{- \frac{1}{m^{1/2}},\frac{1}{m^{1/2}} \}.
	 \end{equation*}
	 Consider the Gradient Flow \eqref{eq:GF3}, then with probability $1-\delta$ over the random initialization of $\{b_r\}_r$, $\{w_r\}_r$ and $\{a_r\}_r$, for every $t \geq 0$
	\[
		\| e(t)\|_2^2 + \| S(t)\|_2^2  \leq \exp(- \tlambda t)\, ( \| e(0)\|_2^2 + \| S(0)\|_2^2)
	\]
	where $\tlambda \geq \frac{\delta }{200\, n^2} $ and $\delta = \min(\alpha \hat{\delta}_1, 2 \beta)$.
    In particular $L(W(t), b(t)) \to 0$ as $t \to \infty$.
\end{theorem}

\subsection{Discussion}

Theorem \ref{thm:bias} establishes that the gradient flow \eqref{eq:GF3} converges to a global minimum and therefore that a wide enough network, randomly initialized and trained with the Sobolev loss \eqref{eq:SobolevTr2}  can interpolate any given function values and directional derivatives.  
We observe that recent works in the analysis of standard training \cite{zou2019improved,oymak2019towards}
have shown that using more refined concentration results and control on the weight dynamics, the polynomial dependence on the number of samples $n$ can be lowered. We believe that by applying similar techniques to the Sobolev training, the dependence of $m$ from the number of samples $n$ and derivatives $k$ can be further improved. 

Regarding the assumptions on the input data, we note that  \cite{oymak2019towards,arora2019fine,du2018gradient} have
shown convergence of gradient descent to a global minimum of the standard $\ell_2$ loss,  when  the 
input points $\{x_i\}_{i=1}^n$ satisfy the separation conditions \eqref{eq:sep_x}. These conditions ensure
that no two input points $x_i$ and $x_j$ are parallel and reduce to \eqref{eq:sep_x2} for a network with bias. While the separation condition \eqref{eq:sep_x} is also required in Sobolev training, the condition \eqref{eq:sep_v} is only required in case of a network without bias as a consequence of its homogeneity.  

Finally, the analysis of  gradient methods for training overparameterized neural networks
with standard losses has been used  to study their inductive bias and ability to learn certain classes of functions (see for example \cite{arora2019fine,bietti2019inductive}). 
Similarly, the results of this paper could be used to shed some light on the 
superior generalization capabilities of networks trained with a Sobolev loss and their use for knowledge distillation.

%% file: sections/main_proof.tex
\section{Proof of Theorem \ref{thm:main_thm}}\label{sec:Proof}

We follow the lines of recent works on the optimization of neural networks in the Neural Tangent Kernel regime \cite{jacot2018neural,chizat2018lazy,oymak2019towards} in particular the analysis of \cite{arora2019fine,du2018gradient,weinan2019comparative}.
We investigate the dynamics of the residuals error $e(t)$ and $S(t)$, beginning with that of the  predictions. Let $\Fb(W(t),x_i) = V_i^T \nabla_x f(W(t),x_i)$, then:
\[
	\begin{aligned}
	\frac{d}{d t} f(W(t),x_i) &= \sum_{r=1}^m  \frac{\pa f(W(t),x_i)}{\pa w_r}^T \frac{\dd w_r(t)}{\dd t}  \\
	&=  \sum_{j=1}^n  A_{ij}(t)(y_j - f(W(t),x_j)) + \sum_{j=1}^n B_{ij}(t) (h_j -  \Fb (W(t), x_j)),
	\end{aligned}
\]
\[
	\begin{aligned}
	\frac{d}{d t} \Fb(W(t),x_i) &= \sum_{r=1}^m \frac{\pa \Fb(W(t),x_i) }{\pa w_r}^T \frac{\dd w_r(t)}{\dd t}   \\
	&=  \sum_{j=1}^n  B_{ji}(t)^T(y_j - f(W(t),x_j)) + \sum_{j=1}^n C_{ij}(t) (h_j - \Fb(W(t), x_j) ),
	\end{aligned}
\]
where we defined the matrices $A(t) = \sum_{r=1}^{m} A_r(t) \in \R^{n \times n}$, $B(t) = \sum_{r=1}^{m} B_r(t) \in \R^{n  \times n \cdot k}$, $C(t) = \sum_{r=1}^{m} C_r(t) \in \R^{n \cdot k \times n \cdot k }$, with block structure:
\[	
	[A_r (t)]_{ij} :=  \frac{\pa f(W(t),x_i)}{\pa w_r}^T \frac{\pa f(W(t),x_j)}{\pa w_r}  = \frac{1}{m} \sigma'(w_r^T x_i) \sigma'(w_r^Tx_j) x_i^T x_j,
\]
\[
	[B_r (t)]_{ij} := \frac{\pa f(W(t),x_i)}{\pa w_r}^T  \frac{\pa \Fb(W(t),x_j)}{\pa w_r} = \frac{1}{m} \sigma'(w_r^T x_i) \sigma'(w_r^Tx_j) x_i^T V_j,
\]
\[
	[C_r (t)]_{ij} :=  \frac{\pa \Fb(W(t), x_i)}{\pa w_r}^T \frac{\pa \Fb(W(t), x_j)}{\pa w_r} = \frac{1}{m} \sigma'(w_r^T x_i) \sigma'(w_r^T x_j) V_i^T V_j.
\]
The residual errors \eqref{eq:residuals} then follow the dynamical system:
\begin{equation}\label{eq:dynamics}
	\frac{d}{d t} \begin{pmatrix}e \\ S \end{pmatrix} = - H(t) \begin{pmatrix}e \\ S \end{pmatrix}
\end{equation}
where $H(t) \in \R^{n(k+1) \times n(k+1)}$ is given by:
\begin{equation*}
	H(t)= \left[
	\begin{array}{c c}
	A(t) & B(t) \\
	
	B(t)^T & C(t)
	\end{array}
	\right].
\end{equation*}
We moreover observe that if we define:
\begin{equation*}
	\Omega_r(t) := [\frac{\pa f(W(t), x_1)}{\pa w_r}, \dots, \frac{\pa f(W(t), x_n)}{\pa w_r}, \frac{\pa \Fb(W(t), x_1)}{\pa w_r}, \dots, \frac{\pa \Fb(W(t), x_n)}{\pa w_r}]
\end{equation*}
and $H_r(t) := \Omega_r(t)^T \Omega_r(t)$, then direct calculations show that $H(t) = \sum_r H_r(t)$ and  $H(t)$ is symmetric positive semidefinite for all $t \geq 0$. 
In the next section we will show that 
$H(t)$ is strictly positive definite in a neighborhood of initialization, 
while in section \ref{subsec:conv}
we will show that this holds for large enough time leading 
to global convergence to zero of the errors.

\subsection{Analysis near initialization}\label{sec:near_init}

In this section we analyze the behavior of the matrix $H(t)$ 
and  the dynamics of the errors $e(t), S(t)$ 
near initialization. We begin by bounding the output  and directional derivatives
of the network for every $t$. 

\begin{lemma}\label{lemma:bndHr}
	For all $t \geq 0$ and $1 \leq i \leq n$, it holds:
	\begin{align*}
	\| \frac{\pa f(W(t), x_i)}{\pa w_r} \|_2 &\leq   \frac{1}{\sqrt{m}} \\
	\| \frac{\pa \Fb(W(t), x_i)}{\pa w_r} \|_2 &\leq  {\sqrt{\frac{k}{m}}} \\
	\| H_r(t) \|_2 &\leq  \frac{n (k+1)}{m}. 
	\end{align*}
\end{lemma}

We now lower bound the smallest eigenvalue of $H(0)$.
\begin{lemma}\label{lemma:H0}
	Let $\delta \in (0,1)$, and $m \geq \frac{32}{{\tlambda}}\, n(k+1)  \ln( n(k +1)/\delta)$ then with probability $1-\delta$ over the random initialization:
	\[
	\lambda_{\text{min}}(H(0)) \geq \frac{3}{4}  \tlambda
	\]
\end{lemma}

We now provide a bound on the expected value of the residual errors at initialization.
\begin{lemma}\label{lemma:bound_init}
	Let $\{w_r \}_r$ and $\{a_r\}$ be randomly initialized as in \eqref{eq:initialize}
	then the residual errors \eqref{eq:residuals} at time zero satisfy with probability at least $1 - \delta$:
	\[
		 \|e(0)\|_2 + \| S(0) \|_2 \leq \frac{2 \sqrt{nk} + \gamma}{\delta^{1/2}}
	\]
	where $\gamma = \|y\|_2 + \|h\|_2$.
\end{lemma}

Next define the neighborhood around initialization:
\[
	\Nhb := \big\{W : \| H(W) - H(W(0))\|_F \leq  \frac{ \tlambda}{4}  \big\}
\]
and the escape time 
\begin{equation}\label{eq:tau0}
    \tau_0 := \inf\{t: W(t) \notin \Nhb\}.
\end{equation} We can now prove 
the main result of this section which characterizes the dynamics of $e(t)$, $S(t)$
and the weights $w_r(t)$ in the vicinity of $t = 0$.
\begin{lemma}\label{lemma:decay}
	Let $\delta \in (0,1)$ and $m \geq \frac{32}{{\tlambda}}\, n(k+1)  \ln( n(k +1)/\delta)$ then with probability $1-\delta$ over the random initialization, for every $t \in [0, \tau_0]$:
	\[
	\| e(t) \|_2^2 + \| S(t)\|_2^2  \leq \exp(- \tlambda t)\, ( \| e(0) \|_2^2 + \| S(0)\|_2^2)
	\]
	and
	\[
	\| w_r(t) - w_r(0) \|_2 \leq \frac{4  }{\tlambda} \sqrt{\frac{k n}{m}} (\|e(0)\|_2 + \| S(0)\|_2 )  =: R
	\]
\end{lemma}	

\begin{proof}
	Observe that if $t \in [0,\tau_0]$, by Lemma \ref{lemma:H0} with probability $1-\delta$:
	\[
		\lambda_{\text{min}}(H(t)) \geq \lambda_{\text{min}}(H(0)) - \| H(t) - H(0)\|_F \geq \frac{\tlambda}{2}.
	\]
	Therefore using \eqref{eq:dynamics} it follows that for any $t \in [0,\tau_0]$:
	\[
		\frac{\dd }{\dd t} \frac{1}{2} \big(\|e(t)\|_2^2 + \| S(t)\|_2^2 \big) \leq -  \frac{\tlambda }{2} \big(\|e(t)\|_2^2 + \| S(t)\|_2^2 \big),
	\]
	which implies the first claim by Gronwall's lemma.
	Next, using \eqref{eq:GF}, the bounds in Lemma \ref{lemma:bndHr} and the above inequality we obtain:
	\[
	\begin{aligned}
	\| \frac{\dd }{\dd t} w_r(t) \|_2 &= \| \frac{\pa}{\pa w_r} L(W(t)) \|_2  \\
	&= \| \sum_{i= 1}^n (y_i - f(W(t),x_i)) \frac{\pa f(W(t),x_i)}{\pa w_r} + \sum_{i=1}^n  \frac{\pa \Fb_i(t)}{\pa w_r}(h_i - \Fb_i(t)) \|   \\
	&\leq \frac{1}{\sqrt{m}} \sum_{i=1}^n |y_i - f(W(t),x_i)| + \sqrt{\frac{k}{m}}  \sum_{i=1}^n \| h_i - \Fb_i(W(t))\| \\
	&\leq 2  \sqrt{\frac{k n}{m}} e^{-\frac{\tlambda}{2} t} (\| e(0)\| + \| S(0)\|).
	\end{aligned}
	\]
	We can therefore conclude by bounding the distance from initialization as:
	\[
		\|w_r(t) - w_r(0) \| \leq \int_0^t 	\| \frac{\dd }{\dd t} w_r(s) \|_2  \dd t \leq \frac{4}{\tlambda} \sqrt{\frac{k n}{m}} (\| e(0)\| + \| S(0)\|)
	\]
\end{proof}

\subsection{Proof of Global Convergence}\label{subsec:conv}

In order to conclude the proof of global convergence, according to Lemma \ref{lemma:decay}, we need only to show that $\tau_0 = \infty$ where 
$\tau_0$ is defined in \eqref{eq:tau0}  . Arguing by contradiction, assume
this is not the case and $\tau_0 < \infty$.
Below we bound  $\| H(\tau_0) - H(0)\|_F$.

Let $Q_{ijr} := | \sigma'(w_r^T(\tau_0) x_i) \sigma'(w_r^T(\tau_0)x_j) - \sigma'(w_r^T(0) x_i) \sigma'(w_r^T(0)x_j) |$,
then from the formulas for $[A_r]_{ij}$, 
$[B_r]_{ij}$ and $[C_r]_{ij}$ in the previous sections, we have:
\[
\begin{aligned}
	|[A(\tau_0) - A(0)]_{ij}| &\leq \frac{1}{m} \sum_{r=1}^{m}  Q_{ijr},\\
	\|[B^T(\tau_0) - B^T(0)]_{ij}\|_2 &\leq \frac{\sqrt{k}}{m} \sum_{r=1}^{m}  Q_{ijr},\\
	\|[C(\tau_0)- C(0)]_{ij} \|_F &\leq \frac{k}{m} \sum_{r=1}^{m}  Q_{ijr}.
\end{aligned}
\]
Let $R>0$ as in in Lemma \ref{lemma:decay}, then 
with probability at least $1 - \delta$ for all $ 1\leq r \leq m$ 
we have $\|w_r(\tau_0) - w_r(0) \| \leq R$. Moreover observe that if  $\|w_r(\tau_0) - w_r(0) \| \leq R$ and 
$\sigma'(w_r(0)^T x_i) \neq \sigma'(w_r(\tau_0)^T x_i)$, then $|w_r(0)^T x_i| \leq R$. Therefore, for any $ i \in [n]$ and $r \in [m]$ we can define the event
$ E_{i,r} = \{|w_r(0)^{T} x_i| \leq R \}$
and observe that:
\[
	\Ibb\{ \sigma'(w_r(0)^T x_i) \neq \sigma'(w_r(\tau_0)^T x_i) \} \leq \Ibb\{ E_{i,r} \} + \Ibb\{\|w_{r}(\tau_0) - w_r(0) \| > R \}.
\]
Next note that $w_r(0)^{T} x_i \sim \mathcal{N}(0,1)$, so that $\PX({E}_{i,r}) \leq \frac{2 R}{\sqrt{2 \pi }}$ and in particular:
\[
	\frac{1}{m} \sum_{r=1}^{m}  \EX[Q_{ijr}] \leq \frac{1}{m} \sum_{r=1}^{m}( \PX[E_{i,r}] + \PX[E_{j,r}]) + 2 \PX[\cup_r \{ \|w_{r}(\tau_0) - w_r(0) \| > R\}] \\
	\leq \frac{4 R}{\sqrt{2 \pi }} + 2 \delta.
\]

By Markov inequality  we can conclude that with probability at least $1-\delta$:
\[
\begin{aligned}
\|A(\tau_0) - A(0)\|_F \leq  \sum_{i,j} |[A(\tau_0)]_{ij} - [A(0)]_{ij} | 
&\leq \frac{4 n^2 }{\sqrt{2 \pi} \delta} R + \frac{2 n^2 }{m},\\
\|B^T(\tau_0) - B^T(0)\|_F \leq \sum_{i,j} \|[B^T(\tau_0)]_{ij} - [B^T(0)]_{ij} )\|_2 
&\leq \frac{4  n^2 \sqrt{k}}{\sqrt{2 \pi} \delta} R  + \frac{2 n^2 }{m},\\
\|C(\tau_0) - C(0)\|_F \leq \sum_{i,j} \|[C(\tau_0)]_{ij} - [C(0)]_{ij} )\|_F 
&\leq \frac{4  n^2 k}{\sqrt{2 \pi} \delta} R + \frac{2 n^2 }{m},
\end{aligned}
\]
and using Lemma \ref{lemma:bound_init} together with the definition of $H$ and $R$ :
\[
	\| H(\tau_0) - H(0)\|_F \leq \frac{16  n^2\, k}{\sqrt{2 \pi} \, \delta} R + \frac{8 n^2 }{m} = \mathcal{O} \Big( \frac{ n^3 k^2 }{\sqrt{m} \delta^{2}} \frac{\max(1, \gamma)}{\tlambda} \Big)
\]
Then choosing $m = \Omega({n^6 k^4 \gamma^2}/(\tlambda^4 \delta^3))$ we obtain
\[
\| H(\tau_0) - H(0)\|_F < \frac{\tlambda}{4}
\]
which contradicts the definition of $\tau_0$ and therefore $\tau_0 = \infty$.

%% file: sections/appendix.tex
\section{Supplementary proofs for Section \ref{sec:near_init}}

In this section we provide the remaining proofs of the results in Section \ref{sec:near_init}.
We begin recalling the following matrix Chernoff inequality (see for example \cite[Theorem 5.1.1]{tropp2015introduction}).
\begin{theorem}[Matrix Chernoff]
	Consider a finite sequence $X_k$ of $p \times p$ independent, random, Hermitian matrices 
	with  $0 \preceq X_k \preceq L I$. 
	Let $X = \textstyle{\sum}_k X_k$, then for all $\epsilon \in [0,1)$ 
	\begin{equation}\label{eq:Chernoff}
	\PX\Big[\lambda_{\text{min}}(X) \leq \epsilon \lambda_{\text{min}}\big(\EX[X]\big) \Big] \leq 
	p e^{-(1-\epsilon)^2 \lambda_{\text{min}}(\EX[X]) /2L} 
	\end{equation}
\end{theorem}
In order to lower bound the smallest eigenvalue of $H(0)$ 
we use Lemma  \ref{lemma:bndHr} together with the previous concentration result.
\begin{proof}[Lemma \ref{lemma:H0}]
	We first note that $\EX[H(0)] = \EX[\sum_r H_r(0)] =  H^{\infty}$, 
	and moreover $H_r(0)$ is symmetric positive semidefinite  with $\lambda_{\text{max}}(H_r) \leq n (k+1)/m$ by  Lemma \ref{lemma:bndHr}.
	Applying then the concentration bound \eqref{eq:Chernoff} with the assumption
	$m \geq \frac{32}{{\tlambda}}\,  n(k+1) \ln( n(k +1)/\delta)$ gives the thesis.
\end{proof}

We next upper bound the errors at initialization.
\begin{proof}[Lemma \ref{lemma:bound_init}]
		Note that for any $x_i$, due the the assumption on the independence of the weights at initialization and the normalization of the data:
		\[
		\EX[(f(W,x_i))^2] = \sum_{r=1}^m \frac{1}{m} \EX[\sigma(w_r^T x_i)^2] \leq 1 
		\]
		and similarly for the directional derivatives
		\[
		\EX[ \|\Fb(W,x_i)\|_2^2 ] =  \EX_{g \sim \mathcal{N}(0,I)} [\|\sigma'(g^T x_i ) V_i^T g\|_2^2 ] \leq  \sum_{j=1}^k \EX[ (v_{i,j}^T g)^2 ] \leq  k.
		\]
		We conclude the proof by using Jensen's and Markov's inequalities.
\end{proof}

%% file: sections/supplementary_proofs.tex
\section{Proof of Proposition \ref{prop:tlambda}}

Consider the $d\times (k+1)$ matrices $\mathbf{X}_i = [x_i, V_i]$, and for $w \in \R^d$ define
\[
    \hat{\psi}_w(x_i) = \sigma'(w^T x_i) \mathbf{X}_i.
\]
and the $d \times (k+1)n$ matrix:
\begin{equation*}
	\widehat{\Omega}(w) = [\hat{\psi}_w(x_1), \dots, \hat{\psi}_w(x_n)] \;  \in \R^{d \times n(k+1)}
\end{equation*}
which corresponds to a column permutation of $\Omega(w)$. 
Next observe that the  matrix   $\widehat{H}^\infty = \EX_{w \sim \mathcal{N}(0,I_d)}[\widehat{\Omega}(w)^T \widehat{\Omega}(w)]$ is similar to $H^\infty$
and therefore has the same eigenvalues. 
In this section we lower bound $\tlambda$ by 
analyzing $\widehat{H}^\infty$.

 We begin recalling some facts about the spectral properties of the products of matrices. 
\begin{definition}[\cite{gunther2012schur}]
	Let $\mathbf{A} = [A_{\alpha \beta}]_{\alpha = 1, \dots, n}^{\beta = 1, \dots, n}$ and $\mathbf{B} = [B_{\alpha \beta}]_{\alpha = 1, \dots, n}^{\beta = 1, \dots, n}$ be  $n p \times np$ matrices in which each block is in ${p \times p}$. 
	Then we define the block Hadamard product of $\mathbf{A} \square\mathbf{B}$
	as the $n p \times np$ matrix with:
	\[
		\mathbf{A} \square\mathbf{B}	:= [A_{\alpha \beta} B_{\alpha \beta} ]_{\alpha = 1, \dots, n}^{\beta = 1, \dots, n}
	\]
	where $A_{\alpha \beta} B_{\alpha \beta}$ denotes the usual matrix product between  $A_{\alpha \beta}$ and $B_{\alpha \beta}$.
\end{definition}
Generalizing Schur's Lemma one has the following regarding the eigenvalues of 
the block Hadamard product of two block matrices.
\begin{proposition}[\cite{gunther2012schur}]\label{prop:prod_had}
		Let $\mathbf{A} = [A_{\alpha \beta}]_{\alpha = 1, \dots, n}^{\beta = 1, \dots, n}$ and $\mathbf{B} = [B_{\alpha \beta}]_{\alpha = 1, \dots, n}^{\beta = 1, \dots, n}$ be  $n p \times np$ positive semidefinite matrices. Assume that  every $p\times p$ block of $\mathbf{A}$ commutes with every $p\times p$ block of $\mathbf{B}$, then:
		\[
			\lambda_{\text{min}}(\mathbf{B} \square\mathbf{A})= \lambda_{\text{min}}(\mathbf{A} \square\mathbf{B}) \geq  \lambda_{\text{min}}(A) \cdot \min_{\alpha}\lambda_{\text{min}}(B_{\alpha \alpha})
		\]
\end{proposition}
We finally recall the following on the eigenvalues of Kronecker product of matrices.
\begin{proposition}[\cite{laub2005matrix}]\label{prop:prod_kron}
	Let $A \in \R^{n \times n}$ with eigenvalues $\{\lambda_i\}$ and $B \in \R^{m \times m}$ with eigenvalues $\{\mu_i\}$, then Kronecker product $A \otimes B$ between $A$ and $B$ has eigenvalues $\{\lambda_i \mu_j \}$.
\end{proposition}
We next define the following random kernel matrix.
\begin{definition}
	Let $w \sim \mathcal{N}(0,I)$ then define the random matrix $\mathcal{M}(w) \in \R^{n \times n}$ with entries $[\mathcal{M}(w)]_{ij}= \sigma'(w^T x_i) \sigma'(w^T x_j)$.
\end{definition}
The next result from \cite{oymak2019towards} establishes positive definiteness of this matrix in expectation, under the separation condition \eqref{eq:sep_x}.
\begin{lemma}[\cite{oymak2019towards}]\label{lemma:mad}
	Let $x_1, \dots, x_d$ in $\R^d$ with unit Euclidean norm and assume that \eqref{eq:sep_x} is satisfied for all $i = 1, \dots d$. Then the following holds: 
	\[
		\EX_{w \sim \mathcal{N}(0,I)} [\mathcal{M}(w)] \succeq \frac{\delta_1 }{100 n^2}
	\]
\end{lemma}
Finally let $\mathbf{X} \in \R^{d \times n(k+1)}$ block matrix with $d\times (k+1)$ blocks $\mathbf{X}_i$.
Thanks to the assumption \eqref{eq:sep_v} the following
result on the Gram matrices $\mathbf{X}_i^T\mathbf{X}_i$ holds.
\begin{lemma}\label{lemma:Gersh}
	Assume that the condition \eqref{eq:sep_v} is satisfied, then for any $i = 1, \dots, n$ we have $\lambda_{\text{min}} (\mathbf{X}_i^T\mathbf{X}_i) \geq 1 - k \delta_2 > 0$.
\end{lemma}
\begin{proof}
	The claim follows by observing that by \textit{Gershgorin’s Disk Theorem}:
	\[
		| \lambda_{\text{min}} (\mathbf{X}_i^T\mathbf{X}_i) - 1| \leq \sum_{ 1 \leq j \leq k} |x_i^T v_{i,j}|\leq k \delta_2.
	\]
\end{proof}
Finally observe that we can write:
\[
	 \widehat{H}^\infty = \EX_{w \sim \mathcal{N}(0,I)}\big[ (\mathbf{X}^T \mathbf{X})\square(\mathcal{M}(w) \otimes I ) \big].
\]
so that Proposition \ref{prop:prod_had}, Proposition \ref{prop:prod_kron}, Lemma \ref{lemma:mad} and Lemma \ref{lemma:Gersh} allow to derive the thesis of Proposition \ref{prop:tlambda}.